\documentclass[twoside]{article}
\usepackage[accepted]{aistats2025}

\usepackage{amsmath,amsfonts,bm}

\DeclareMathOperator*{\RC}{RC}

\DeclareMathOperator*{\rcut}{RatioCut}
\newcommand{\Ea}[1]{\E\left[#1\right]}
\newcommand{\Eb}[2]{\E_{#1}\left[#2\right]}
\newcommand{\Ppar}[1]{\operatorname{Pr}\lrp{#1}}
\newcommand{\Prob}[1]{\operatorname{Pr}\lrp{#1}}
\newcommand{\by}[1]{{\frac{1}{#1}}}

\newcommand{\comp}[1]{\overline{#1}}
\newcommand{\cosine}[1]{\mathop{cosine}\lrp{#1}}
\newcommand{\defeq}{\stackrel{\text{def}}{=}}
\newcommand{\diag}{\mathop{\mathrm{diag}}}
\newcommand{\kl}[2]{D_{\rm KL}(#1 \ \| \ #2)}

\newcommand{\lrb}[1]{\left[#1\right]}
\newcommand{\lrcb}[1]{\left\lbrace#1\right\rbrace}
\newcommand{\lrp}[1]{\left(#1\right)}

\newcommand{\normtwo}[1]{\left\lVert#1\right\rVert}

\newcommand{\ov}[1]{\overline{#1}}
\newcommand{\stopgrad}{\operatorname{sg}}

\def\eqref#1{equation~\ref{#1}}

\def\1#1{\bm{1}_{#1}}

\def\rZ{{\textnormal{Z}}}

\def\rva{{\mathbf{a}}}

\def\rvf{{\mathbf{f}}}

\def\rvr{{\mathbf{r}}}

\def\vone{{\bm{1}}}

\def\vc{{\bm{c}}}

\def\vf{{\bm{f}}}

\def\vl{{\bm{l}}}

\def\vp{{\bm{p}}}

\def\vy{{\bm{y}}}

\def\evp{{p}}

\def\mD{{\bm{D}}}

\def\mF{{\bm{F}}}

\def\mI{{\bm{I}}}

\def\mL{{\bm{L}}}

\def\mP{{\bm{P}}}

\def\mW{{\bm{W}}}

\def\mSigma{{\bm{\Sigma}}}
\def\mPi{{\bm{\Pi}}}

\DeclareMathAlphabet{\mathsfit}{\encodingdefault}{\sfdefault}{m}{sl}
\SetMathAlphabet{\mathsfit}{bold}{\encodingdefault}{\sfdefault}{bx}{n}

\def\gC{{\mathcal{C}}}

\def\gE{{\mathcal{E}}}

\def\gG{{\mathcal{G}}}
\def\gH{{\mathcal{H}}}
\def\gI{{\mathcal{I}}}

\def\gK{{\mathcal{K}}}
\def\gL{{\mathcal{L}}}

\def\gU{{\mathcal{U}}}
\def\gV{{\mathcal{V}}}

\def\sA{{\mathbb{A}}}

\def\sC{{\mathbb{C}}}

\def\sI{{\mathbb{I}}}

\def\sP{{\mathbb{P}}}

\def\sR{{\mathbb{R}}}

\def\emP{{P}}

\def\emW{{W}}

\newcommand{\E}{\mathbb{E}}

\newcommand{\R}{\mathbb{R}}

\DeclareMathOperator{\Tr}{Tr}

\usepackage{amsmath}
\usepackage{amsthm}
\usepackage{booktabs}
\usepackage{multirow}
\usepackage{hyperref}
\usepackage{url}
\usepackage{cleveref}
\usepackage{mathtools}
\usepackage{stackengine}
\usepackage{thmtools}
\usepackage{thm-restate}
\usepackage{amsmath}
\usepackage{algorithm}
\usepackage{algorithmic}
\usepackage{authblk}

\usepackage[round]{natbib}

\crefname{table}{table}{tables}
\Crefname{table}{Table}{Tables}

\newtheorem{corollary}{Corollary}

\declaretheorem[name=Lemma,numberwithin=section]{lemma}

\begin{document}

\runningauthor{Ayoub Ghriss, Claire Monteleoni}

\twocolumn[
    \aistatstitle{Deep Clustering via Probabilistic Ratio-Cut Optimization}
    \aistatsauthor{
        Ayoub Ghriss\textsuperscript{1}\\ayoub.ghriss@colorado.edu 
        \And 
        Claire Monteleoni\textsuperscript{1,2}\\cmontel@colorado.edu}
    \aistatsaddress{
        \textsuperscript{1}Department of Computer Science, University of Colorado Boulder \\
        \textsuperscript{2} INRIA Paris
    }
]
\begin{abstract}

	We propose a novel approach for optimizing the graph ratio-cut by modeling the binary assignments as random variables. We provide an upper bound on the expected ratio-cut, as well as an unbiased estimate of its gradient, to learn the parameters of the assignment variables in an online setting. The clustering resulting from our probabilistic approach (PRCut) outperforms the Rayleigh quotient relaxation of the combinatorial problem, its online learning extensions, and several widely used methods. We demonstrate that the PRCut clustering closely aligns with the similarity measure and can perform as well as a supervised classifier when label-based similarities are provided. This novel approach can leverage out-of-the-box self-supervised representations to achieve competitive performance and serve as an evaluation method for the quality of these representations.

\end{abstract}

\section{INTRODUCTION}

Unsupervised learning is based on the premise that labels are not necessary for the
training process, particularly for clustering tasks where samples that are highly
similar are grouped together. Various clustering
algorithms~\citep{clusteringsurvey} have been proposed in the context of machine
learning and data mining. The K-Means algorithm~\citep{lloyd} and ratio-cut
partitioning~\citep{ratiocut} were among the first approaches to address the
clustering problem. These methods were further refined and extended beyond binary
partitioning and Euclidean distances, and they were even shown to be fundamentally
equivalent in specific settings~\citep{kernelkmeans}.

The recent advances in generative models and self-supervised representation
learning have produced powerful embeddings that capture the similarity between the
original data samples. This makes leveraging these similarity measures to achieve
effective clustering more relevant than ever, as it can serve as pseudo-labels for
pre-training classifiers or eliminate the need for a decoder~\citep{decoder} in the
training of these generative models. Therefore, it is highly advantageous to
develop a method that can transform the similarity information into clustering of
equal quality. An efficient extension of spectral clustering to stochastic gradient
descent would facilitate the conversion of learned embeddings to cluster
assignments without relying on the full dataset or large batches to accurately
approximate the graph structure of the embedding space.

Several clustering methods have been applied to data streams to provide weak
signals for the pre-training of neural networks. Meanwhile, other approaches have
been proposed to utilize deep learning specifically for clustering, particularly
with autoencoders and generative neural networks~\citep{vade}. Generally, these
methods can be categorized into contrastive learning~\citep{contrastive}, where
samples are grouped based on pairwise distances in a large batch, or generative
models such as variational autoencoders that use a prior with an auxiliary variable
as the cluster assignment. The former methods do not fully capture the global
structure of the clusters, while the latter may suffer from overfitting the prior
due to the complexity of the neural network, with limited options to prevent this
without relying on labels. Spectral clustering avoids these issues by clustering
the data based on global similarities between clusters rather than just pairs of
samples.

Spectral clustering has long resisted attempts to extend it to parametric learning,
primarily due to the challenge of handling the spectral decomposition of large
matrices. Since it is based on the spectral decomposition of the relaxed ratio-cut
rather than the combinatorial version, it requires the projection of the data into
a principal eigenspace. A clustering algorithm is then applied to these
projections, further making the final performance dependent on the chosen
clustering algorithm.

In this paper, we develop a novel method to optimize the ratio-cut without relying
on the spectral decomposition of the Laplacian matrix by employing a probabilistic
approach that treats the cluster assignment as random variables. Furthermore, we
utilize neural networks to parameterize the assignment probabilities and address
the clustering problem in an online manner using stochastic gradient descent. The
result is an online learning algorithm that achieves a better ratio-cut objective
than the memory-intensive spectral method applied to the full graph Laplacian. We
also demonstrate that our approach achieves comparable performance to a supervised
classifier when utilizing supervised similarity (i.e., two samples are considered
similar if they share the same label). This showcases a high fidelity between the
resulting clustering and the given similarity measure. The proposed algorithm can
also serve as a drop-in replacement for any other clustering method used in the
pre-training of text or speech transformers, which opens up several opportunities
for enhancing the effectiveness of pre-training for downstream tasks.

\section{BACKGROUND}
In this section, we succinctly present the relevant elements related to the notion
of graph ratio-cut. The curious reader may refer to \citet{spectralclustering} for
a more detailed account of spectral clustering and other types of graph cuts.

Let $\gG=(\gV,\gE,\gK)$ be an undirected weighted graph where $\gV \defeq \lrcb{v_i
\mid 1 \leq i \leq n} \subset \mathbb{R}^p$ is the set of vertices, $\gE \subset
\gV \times \gV$ is the set of edges $e_{ij}$ linking vertices $v_i$ and $v_j$, and
$\gK: \gV \times \gV \to \R^+$ is a symmetric non-negative kernel. Let $\mW$ be the
symmetric $n \times n$ adjacency matrix where $\emW_{ij} = \gK(v_i, v_j)$. The
degree of the vertex $v_i$ is $d_i=\sum_j \emW_{ij}$, and the degree matrix
$\mD\defeq\diag{(d_1,\ldots,d_n)}$.

Let $k \geq 2$ and $\gC_k=\lrcb{\sC_\ell|1\leq\ell\leq k}$ a partitioning of the
graph $\gG$ into $k$ disjoint clusters. We shall represent the subset $\sC_\ell
\subset \gV$ using the binary assignment vector $\1{\sC_\ell} \in \lrcb{0,1}^n$
where $\1{\sC_\ell}(i)=1$ if and only if $v_i \in \sC_\ell$. The size of $\sC_\ell$
is measured using its cardinality $|\sC_\ell| = \sum_{i=1}^n \1{\sC_\ell}(i)$, and
we denote by $\vf^{(\ell)} \defeq \by{\sqrt{|\sC_\ell|}} \1{\sC_\ell}$ its
\textit{ratio assignment} when $|\sC_\ell| >0$.

The \textit{ratio-cut} for $\gC_k$ is defined as:
\begin{align} \label{eq:ratiocut}
	\rcut(\gC_k) & \defeq \by{2}\sum_{\ell=1}^{k} \by{|\sC_\ell|}
	\sum_{i,j\in \sC_\ell\times \comp{\sC_\ell}} \emW_{ij}        \\
	             & = \by{2}\sum_{\ell=1}^{k} \by{|\sC_\ell|}
	\1{\sC_\ell}^\top \mW \lrp{\bm{1}_n-\1{\sC_\ell}},
\end{align}
where $\comp{\sA}$ denotes $\gV\backslash \sA$, the complement of $\sA$ in $\gV$.

We define the unnormalized Laplacian matrix as $\mL_{un}\defeq \mD-\mW$, which can
be used to express the ratio-cut of $\gC_k$ as:
\begin{align}
	\label{eq:rawlap}
	\rcut(\gC_k)
	=\by{2}\Tr\lrb{\mF_{\gC_k}^\top \mL_{un} \mF_{\gC_k}},
\end{align}
where $\mF_{\gC_k} \defeq \lrb{\vf^{(1)},\ldots \vf^{(k)}}\in \R^{n \times k}$ is the
ratio assignments matrix.

Since the clusters should be disjoint and different from the naive partitioning
$\lrcb{\gV,\emptyset,\ldots,\emptyset}$, we can express these constraints as
$\mF_{:,\ell}^\top\1{V}\neq 1$ for all $1\leq \ell\leq k$.

The optimization of the ratio-cut is then equivalent to minimizing a Rayleigh
quotient on $\{0,1\}^{n\times k}$. Solving~\cref{eq:rayleighquo} on the set
$\{0,1\}^n$ is generally an NP-hard problem. The optimization is hence relaxed so
that the unknown vectors are in the unit sphere $S^{(n)} = \lrcb{x \in \R^n,
\normtwo{x}^2 =1}$. The objective is then:

\begin{align}
	\label{eq:rayleighquo}
	\begin{aligned}
		 & \underset{\mF}{\text{minimize}} &  & \Tr(\mF^\top \mL \mF )                                      \\
		 & \text{subject to}               &  & \mF^\top \mF = \mI_k \text{ and } \mF_{:,j}^\top\1{V}\neq 1
	\end{aligned}
\end{align}

The minimization of the Rayleigh quotient under the outlined constraints yields the
$k$ smoothest eigenvectors of the Laplacian matrix (excluding the trivial first
eigenvector $\1{V}$). Vertices within the same cluster are anticipated to have
similar projections onto the solutions of the relaxed problem. As we ascend the
Laplacian spectrum, the projections onto the eigenvectors encapsulate more specific
(higher frequency) features. Subsequently, the binary assignments are determined
through $k$-means clustering~\citep{standardspectral} of the relaxed problem's
solution.

In this paper, we introduce a novel approach to optimizing~\cref{eq:rayleighquo}
that circumvents the necessity for spectral decomposition of extensive matrices or
kernels, thereby sidestepping the relaxation of the problem into Euclidean space.
Rather, we propose to relax the problem into an optimization over a simplex through
the parameterization of the cluster assignment distribution.

\section{RELATED WORK}\label{sec:related}
There have been many attempts to extend Spectral Clustering to streaming
settings~\citep{streamclustering}, infinitely countable datasets, or continuous
input spaces. For instance, \citet{streamspec} proposed an extension to
non-parametric Spectral Clustering for data streams, particularly when the kernel
similarity is bilinear or can be approximated linearly. With the arrival of each
new stream batch, a batch-based Singular Value Decomposition (SVD) is used to embed
the stream and realign a facility set consisting of anchor points. These anchor
points are then utilized to run a batch $K$-means algorithm~\citep{fastkmeans} to
cluster the new points.

However, research on extending Spectral Clustering to the domain of parametric
learning, where the optimization variable is a mapping of the input space to the
eigenspace, is limited. This is primarily due to the sensitivity of the
eigendecomposition~\citep{sensitiveigen} and the proven difficulty of
eigendecomposition in non-parametric settings when the similarity kernel takes
specific forms~\citep{sparsenystrom}.

One of the early attempts in this direction was made in the context of Deep
Reinforcement Learning~\citep{deeprloptions}, where the Markov Decision Process
(MDP) is decomposed into sub-tasks that navigate the representation
space~\citep{eigenoption}. These sub-tasks, also called
\textit{options}~\citep{optiondiscover}, correspond to eigenvectors associated with
the largest eigenvalues of the graph Laplacian, where the vertices represent
elements of the state space and the similarity is the likelihood of the agent
moving from one state to another. Extending such an approach to a clustering
setting would require constructing an MDP for which the correspondence between
eigenvectors and options holds.

SpectralNets~\citep{spectralnet} were perhaps the first attempt at explicitly
training a neural network to approximate eigenvectors of the Laplacian for
large-scale clustering. The learned mapping is a neural network $N_\theta : \R^p
\rightarrow \R^k$, parameterized by $\theta$, that maps the input vertices to their
projection on the subspace spanned by the $k$-smoothest eigenvectors of the
Laplacian. The network is constructed to ensure that the output features are
approximately orthonormal:
\[
	\by{n}\sum_{i=1}^{n} {N_\theta(v_i)}^\top {N_\theta(v_i)} \approx I_k
\]
We denote by $\hat{N}_\theta$ the neural network without the last
  orthogonalization layer, which is based on the Cholesky decomposition of the
  estimated covariance matrix $\mSigma_\theta \in \R^{k \times k}$:
\[
	\mSigma_\theta = \by{n}\sum_{i=1}^{n} {\hat{N}_\theta(v_i)}^\top {\hat{N}_\theta(v_i)}.
\]
SpectralNets approach is then based on the minimization of:
\begin{align}
	\label{eq:spectralneloss}
	\gL_{spectral}(\theta)\defeq
	\sum_{i,j=1}^{n} \emW_{ij} \normtwo{N_\theta(v_i) -{N_\theta(v_j)}}^2
\end{align}

The matrix $\mPi_\theta \in \sR^{k \times k}$ serves as the parametric counterpart
of $\mF^\top \mL \mF$ derived from \cref{eq:rayleighquo}. Training the neural
network equates to an alternating optimization scheme, where one batch is utilized
to estimate $\mSigma_\theta$, while the other is employed to minimize the loss
$\gL_{spectral}(\theta)$.

During training, should the input exhibit noise, the SpectralNet's output may
converge to a constant mapping, resulting in the spectral loss approaching zero.
Consequently, this approach encounters numerical instability in high-dimensional
data scenarios. One common strategy to circumvent this involves offsetting the
matrix $\mSigma_\theta$ by a scaled identity and establishing an appropriate
stopping criterion. However, the resulting algorithm remains highly susceptible to
input and similarity noise. This sensitivity becomes particularly evident when
benchmarked against non-parametric Spectral Clustering. SpectralNets achieve
competitive performance only when the input consists of the code space from another
well-performing pre-trained variational encoder~\citep{vade}.

\textit{Spectral Inference Networks} (Spin)~\citep{spin} adopt a distinct strategy
for parametric optimization of the Rayleigh quotient. As defined previously, the
loss function takes the form:

\begin{align*}
	\gL_{spin}(\theta) = \Tr\lrp{{\mSigma_\theta}^{-1}\mPi_\theta}
\end{align*}

While primarily utilized within the domain of Deep Reinforcement Learning,
targeting the eigenfunctions associated with the largest eigenvalues, this approach
can be adapted for loss minimization rather than maximization.

Upon differentiation~\citep{matrixcook}, the full gradient can be expressed as:
\begin{align*}
	\Tr\lrp{{\mSigma_\theta}^{-1}\nabla_\theta\mPi_\theta} -
	\Tr\lrp{{\mSigma_\theta}^{-1}(\nabla_\theta{\mSigma_\theta}){\mSigma_\theta}^{-1}\mPi_\theta}
\end{align*}

The Spin algorithm estimates $\mSigma_\theta$ and $\nabla_\theta \mSigma_\theta$
utilizing moving averages. It addresses the issue of overlapping eigenvectors by
modifying the gradient to ensure sequential independence in updates: the gradients
of the first $l$ eigenfunctions are made independent of gradients from other
components ${l+1,\ldots, k}$. This alteration significantly slows down training, in
addition to the computational burden of computing $\nabla_\theta\mSigma_\theta$,
which necessitates calculating $k^2$ backward passes and storing $k^2 \times
\text{size}(\theta)$, thereby increasing memory and computational costs.

We assert that the performance of these attempts is ultimately hindered by the
difficulty of the eigendecomposition of large kernels. Therefore, we take a
different approach to optimizing the graph ratio-cut that circumvents the spectral
decomposition.

\section{PROPOSED METHOD}
In our probabilistic approach, the minimization of the ratio-cut is addressed
differently. Instead of the deterministic assignments $\vone_{\sC_\ell}$, we use
the random assignment vector $\rva^{(\ell)} \in \lrcb{0,1}^n$ under the assumption
that $(\rva_i^{(\ell)})_i$ are independent random variables such that:

\[
	\Prob{v_i \in \sC_\ell} = \Prob{\rva^{(\ell)}_i = 1} \defeq  \emP_{i,\ell},
\]
where the rows of the matrix $\mP\in[0,1]^{n\times k}$ sum to $1$:
  $\sum_{\ell=1}^k \emP_{i,\ell} = 1$ for all $i\in\lrcb{1,\ldots,n}$.

The random assignment for each vertex $i$ follows a categorical distribution of
parameter $\mP_{i,:}$ and the random clustering $\gC_k$ is thus parameterized by
$\mP\in[0,1]^{n\times k}$. We define the random ratio-assignment vector
$\rvf^{(\ell)}$ for cluster $\sC_\ell$ as:
\[
	\rvf^{(\ell)} = \by{\sqrt{\widehat{|\sC_\ell|}}}\rva^{(\ell)},
\]
where $\widehat{|\sC_\ell|} = \sum_{i=1}^{n}\rva^{(\ell)}_i$, and all the
  elements of $\rvf^{(\ell)}$ are in the interval $\lrb{0,1}$ with the convention
  $\frac{0}{0} = 1$.

The stochastic counterpart of the quantity introduced earlier in \Cref{eq:ratiocut}
is:

\begin{align}
	\label{eq:stochasticrcut}
	\widehat{\rcut(\gC_k)}
	 & = \by{2}\sum_{\ell=1}^k \sum_{1 \leq i,j\leq n} \emW_{ij} (\rvf^{(\ell)}_i-\rvf^{(\ell)}_j)^2.
\end{align}

\Cref{eq:rawlap} can also be extended to the stochastic setting similarly.

The next step is to compute the expected ratio-cut for a clustering $\gC_k$
parameterized by $\mP$. Without loss of generality, we only have to compute
$\Ea{(\rvf^{(\ell)}_1-\rvf^{(\ell)}_2)^2}$ as a function of $\mP$ due to the
linearity of the expectation. We prove the following result
in~\Cref{appendix:diffexpect}:

\begin{restatable}[Difference expectation]{lemma}{diffexpect}
	\label{lemma:diffexpect}
	Let $\sC$ be a subset of $\gV$ and $\rva$ its random assignment vector
	parameterized by $\vp\in[0,1]^n$. Let $\rvf$ be its random ratio-assignment
	vector. Then we have the following:
	\begin{align*}
		\Ea{(\rvf_1-\rvf_2)^2} = (\evp_1 + \evp_2 -2 \evp_1 \evp_2)\Ea{\by{1 + \widehat{|\sC^{\perp(1,2)}|}}},
	\end{align*}
	where $\sC^{\perp(i,j)} = \sC \backslash \lrcb{v_i, v_j}$
\end{restatable}

To avoid excluding pairs $(i,i)_{1\leq i \leq n }$ from the summation in
~\Cref{eq:stochasticrcut} each time, we will henceforth assume that $\emW_{ii} = 0$
for all $i \in \lrcb{1,\ldots,n}$.

The random variable $\widehat{|\sC^{\perp(1,2)}|} = \sum_{i=3}^n \rva_i$ is the sum
of (n-2) independent (but not identical) Bernoulli random variables, also known as
a Poisson binomial. It coincides with the binomial distribution when the Bernoulli
random variables share the same parameter. See \Cref{appendix:poissonexpect} for a
compilation of properties of this distribution including the proof
of~\Cref{lemma:poissonexpect}.

\begin{restatable}[Poisson binomial expectation]{lemma}{poissonexpect}
	\label{lemma:poissonexpect}
	Let $\rZ$ be a Poisson binomial random variable of parameter
	$\bm{\alpha}=(\alpha_1,\ldots,\alpha_m)\in[0,1]^m$. Then we have:
	\[
		\Ea{\by{1 + \rZ}}  = \int_{0}^{1} \prod_{i=1}^{m} (1-\alpha_i t)dt
	\]
\end{restatable}

We can now express the expected ratio-cut by combining the results from
\Cref{lemma:diffexpect} and \Cref{lemma:poissonexpect}.

\begin{restatable}[Ratio-cut expectation]{thm}{ratiocutexp}
	\label{thm:ratiocutexp}
	The expected ratio-cut of the random clustering $\gC_k$ parameterized by
	$\mP\in[0,1]^{n\times k}$ can be computed as $\Ea{ \widehat{\rcut(\gC_k)} } =
		\sum_{\ell=1}^k \RC(\mP_{:,\ell}) $ where:
	\begin{align*}
		\RC(\vp)       & = \by{2} \sum_{i,j=1}^n \emW_{ij}\lrp{\evp_i + \evp_j -2\evp_i\evp_j} \sI(\evp, i, j) \\
		\sI(\vp, i, j) & \defeq \int_{0}^{1} \prod_{m\neq i,j}^{n} (1-\evp_m t)dt
	\end{align*}
\end{restatable}
The expression for $\RC(\vp)$ can be further simplified by consolidating the terms linear in $\vp$. However, we retain the current formulation as it will prove useful when sampling random pairs $(i, j)$ in the implementation.

We will drop the cluster index $\ell$ in the next sections and use $\vp \in
[0,1]^n$ as the parameter of the random assignment. The goal now is to optimize
$\RC(\vp)$.

\subsection{Computation of the expected ratio-cut}\label{subsec:computeall}
A straightforward method to estimate the expected ratio-cut from
\Cref{thm:ratiocutexp} involves discretizing the interval $[0,1]$ to approximate
the integral. By using a uniform discretization with a step size of $\by{T}$ for
$T>0$, we can apply the formula:
\[
	\sI(\vp, i, j) = \by{T} \sum_{t=1}^{T} \prod_{m\neq i,j} (1-\evp_i \frac{t}{T})
\]
The quality of the approximation depends on $T$, but since the integrated
  function is polynomial in $t$, we have the luxury of using a weighted
  discretization scheme that ensures the exact computation of the integral using
  the following quadrature method proven in~\Cref{appendix:integralcompute}.

\begin{restatable}[Integral computation via quadrature]{lemma}{integralcompute}
	\label{lemma:integralcompute}
	Let $m>0$ be an integer and $c_m \defeq \lfloor\frac{m}{2}\rfloor + 1$ (the
	smallest integer such that $2c_m\geq m+1$), then there exist $c_m$ tuples
	$(s_q,t_q)_{1\leq q \leq c_m}$ such that:
	\[
		\int_{0}^{1} \prod_{i=1}^{m} (1-\evp_i t)dt = \sum_{q=1}^{c_m} s_q \prod_{i=1}^{m} (1-\evp_i t_q),
	\]
	for all $\vp\in [0,1]^m$.
\end{restatable}

The computation of $\sI(\vp, i,j)$ for a given pair $(i,j)$ costs $O(n^2)$. We can
still obtain the total $\RC(\vp)$ in $O(c_n n^2)$ instead of $O(n^4)$ by computing
the $c_n$ quantities $\lrp{\sum_{m=1}^n\log\left(1-\evp_{m} t_q\right)}_{1\leq
q\leq c_n}$ in advance. Another challenge encountered when using the quadrature is
the instability of the computations due to the potentially large number of elements
in the product. Even with a tight approximation using batch-based estimates, it
still costs $O(b^3)$, where $b$ is the batch size. Furthermore, the higher the
number of clusters, the larger the batch size should be to obtain a good estimate
of the integral for different clusters. These challenges are to be expected since
the original combinatorial problem is generally NP-hard. See
\Cref{appendix:batchestimate} for a detailed analysis of the batch-based estimate.
Instead, we prove in~\Cref{appendix:integralupperbound} that the expected ratio cut
can be upper-bounded by a much more accessible quantity, as shown
in~\Cref{lemma:integralupperbound}.

\begin{restatable}[Integral upper-bound]{lemma}{integralupperbound}
	\label{lemma:integralupperbound}
	We adopt the same notation of \Cref{lemma:poissonexpect} where
	$\bm{\alpha}=(\alpha_1,\ldots,\alpha_m)\in[0,1]^m$, and we assume that
	$\ov{\bm{\alpha}}\defeq\by{m}\sum_{i=1}^m \alpha_i > 0$. Then we have:
	\[
		\int_0^1 \prod_{i=1}^{m}(1-\alpha_i t)dt \leq \by{(m+1)\ov{\bm{\alpha}}}.
	\]
\end{restatable}
The assumption $\ov{\bm{\alpha}} > 0$ is not necessary if we extend the
property to $\R\cup\lrcb{+\infty}$. We can now use \Cref{lemma:integralupperbound}
to upper-bound $\RC(\vp)$.

\begin{restatable}[$\RC$ upper-bound]{lemma}{rcupper}\label{lemma:rcupper}
	Using the notations of \Cref{thm:ratiocutexp} and
	\Cref{lemma:integralupperbound}, we have:
	\[
		\RC(\vp) \leq \frac{e^2}{2n}\by{\ov{\vp}} \sum_{1 \leq i,j\leq n}
		w_{ij}(p_i + p_j  -2 p_i p_j)
	\]
\end{restatable}
\paragraph{Interpretation}

Before we proceed with determining an unbiased gradient estimator for the bound in
\Cref{lemma:rcupper}, we seek to understand the significance of the various
quantities introduced previously. In contrastive learning, the objective can be
written as $-2\sum_{ij} \emW_{ij} \evp_i \evp_j$, whose minimization aims to assign
highly similar samples to the same cluster. The objective in ratio-cut can be
expressed as $\sum_{ij} \emW_{ij} \lrb{\evp_i(1-\evp_j) + \evp_j(1-\evp_i)}$, which
is minimized when dissimilar samples are separated into different clusters, in
alignment with the initial goal of minimizing the cut. It is worth noting that the
minimum of the term $\emW_{ij} \lrb{\evp_i(1-\evp_j) + \evp_j(1-\evp_i)}$ is
$\emW_{ij}$, and it is achieved when either $v_i$ and $v_j$ are in different
clusters.

\Cref{lemma:poissonexpect} generalizes the
scaling by the inverse of the size of a cluster $\sC$ to the probabilistic setting.
Let's assume that $\vp \in \lrcb{0,1}^n$ (deterministic assignment), then
$(\evp_i+\evp_j-2\evp_i\evp_j)=1$ if and only if $v_i$ and $v_j$ are in different
clusters. In such case, we can easily prove the following:
$$
	\Ea{\by{1 + \widehat{|\sC|}}} = \int_{0}^{1} \prod_{i=1,\evp_i = 1}^{n} (1-t) dt = \frac{1}{1+n\ov{\vp}}.
$$
As $n$ becomes sufficiently large, the quantity $\frac{1}{n}\by{\ov{\vp}}$ from
\Cref{lemma:rcupper} approximates $\frac{1}{1+n\ov{\vp}}$ and, therefore, also
approximates $\Ea{\by{1 + \widehat{|\sC|}}}$. Consequently, the bounds we've seen
so far are tight in the deterministic case.
\subsection{Stochastic gradient of the objective}

We will drop the scaling factor $\by{2n}$ when upper bounding $\RC(\gC_k)$, we can
thus succinctly write the following upper bound for the expected ratio-cut with
$\ov{\mP} = \diag \lrp{\ov{\mP_{:,1}},\ldots,\ov{\mP_{:,k}}}$ as $\RC(\gC_k) \leq
\gL_{rc}(\mW, \mP)$ such that:
\begin{align}
	\label{eq:quad}
	\gL_{rc}(\mW, \mP)
	 & = \sum_{\ell=1}^k\sum_{i,j=1}^n
	\by{\ov{\mP_{:,\ell}}}\emW_{ij}\lrp{\emP_{i,\ell} + \emP_{j,\ell} -2 \emP_{i,\ell}\emP_{j,\ell}} \nonumber \\
	 & =\Tr(\ov{\mP}^{-1}(\vone_{n,k}-\mP)^\top\mW \mP ),
\end{align}

Let us examine the derivative of $\gL_{rc}$ of a single cluster with respect to
$\evp_i$,$ \dot{p_i}= \frac{d \gL_{prcut}}{d\evp_i} $:
\begin{align}
	\label{eq:batchgradient}
	\dot{p_i}\propto \by{\ov{p}^2} \sum_{i,j=1}^{n}\emW_{ij}
	\Big[(1 -2p_j)\ov{p}                                       
		\left. - \by{n}(p_i + p_j  -2 p_i p_j) \right],
\end{align}

In order to obtain an unbiased estimate for the gradient, we need to acquire an
accurate estimate of $\ov{\vp}$. We do this in the online setting by computing the
moving average:
\[
	{\ov\vp_t}^{ma} = (1-\beta_t){\ov\vp_{t-1}}^{ma} +\beta_t \ov{\vp}_t,
\]
such that $\beta_t = \frac{\beta}{t}$ and $\ov{\vp}_t$ is a batch-based estimate
  of $\ov{\vp}$ at time $t$. Then, the unbiased gradient can be obtained by
  back-propagating $\lrb{\stopgrad(\dot{p_i})p_i}$, with $\stopgrad$ being the
  gradient-stopping operator.

\begin{algorithm}[t]
	\caption{Probabilistic Ratio-Cut (PRCut) Algorithm}\label{alg:prcut}
	\begin{algorithmic}[1]
		\REQUIRE Dataset $(v_i)$, Similarity kernel $\gK$, batch size $b$,
		encoder $N_\theta$ parameterized by $\theta$, number of clusters $k$,$\ov{\mP}_0=\by{k}\vone_k$,$\beta>0$,
		polytope regularization weight $\gamma$, $t= 0$.
		\WHILE{not terminated}
		\STATE $t\gets t+1$
		\STATE Sample left batch $S_l$ of size $b$
		\STATE Sample right batch $S_r$ of size $b$
		\STATE Compute $\mW = \gK(S_l,S_r)$
		\STATE compute $\mP_\theta^l, \mP_\theta^r \gets N_\theta(S_l),N_\theta(S_r)\in{\sR^{b\times k}}$
		\STATE Update $\ov{\mP}_t \gets (1-\beta_t)\ov{\mP}_{t-1} +\frac{\beta_t}{2}\lrp{\ov{\mP_\theta^l} + \ov{\mP_\theta^r}}$
		\STATE Compute $\dot{\mP_\theta^l}$ and $\dot{\mP_\theta^r}$ using \Cref{eq:batchgradient}
		\STATE Back-propagate $\Tr\lrb{\mP_\theta^l \stopgrad(\dot{\mP_\theta^l})^\top +\mP_\theta^r\stopgrad(\dot{\mP_\theta^r})^\top}$
		\STATE Back-propagate $\gamma\kl{\ov{[\mP_\theta^l, \mP_\theta^r]}}{\by{k}\mI_k}$
		\STATE Use the accumulated gradients $g_t$ to update $\theta$
		\ENDWHILE
	\end{algorithmic}
\end{algorithm}
\subsection{Regularization}
As we aim to optimize $\gL_{rc}(\mW, \mP)$ using gradient descent, the main
challenge our method faces is that the assignments may collapse to a single cluster
$\sC_m$ with a probability close to 1. This translates to
$\lrp{\ov{\mP_{:,\ell}}}_{\ell\neq m}\approx 0$, which renders the gradient updates
highly unstable. \citet{sinkhorn} have adopted the constraint that $\mP$ belongs to
the polytope of distributions $\gU_k$, defined as:
\[
	\gU_k=\lrcb{P\in \sR_{+}^{n\times k} | P^\top \vone_{n} = \by{k}\vone_{k} },
\]

Instead of restricting the clusters to be equally likely through the Bregman
projection into $\gU_k$, we only encourage such behavior using Kullback-Leibler
divergence regularization. By selecting the appropriate regularization weight
$\gamma$, the additional term will ensure that the likelihood of each of the $k$
clusters, $\ov{\mP}$, exceeds a certain threshold $\delta(\gamma) > 0$ for the
optimal $\mP$. This approach does not necessarily imply that all the clusters will
be utilized, as the assignment probability for cluster $\sC_\ell$ could be
significantly above $\delta$ without any sample being more likely assigned to
cluster $\sC_\ell$. The final objective that we aim to optimize is:
\[
	\gL_{prcut}(\mW,\mP) = \gL_{rc}(\mW,\mP) + \gamma\kl{\ov{P}}{\by{k}\vone_k}
\]

\subsection{Similarity measure}
We have assumed thus far that the kernel $\gK$ is provided as input. We also assert
that the performance of any similarity-based clustering ultimately depends on the
quality of the similarity function used. The simplest kernel we can use is the
adjacency matrix for a $k$-nearest neighbors graph, where $\gK(v_i,v_j)=1$ when
either $v_i$ or $v_j$ is among the $k$ nearest neighbors of each other (to ensure
that the kernel is symmetric). We can also use the \textit{Simple} (SimCLR) or the
\textit{Symbiotic} (All4One) contrastive learning methods~\citep{simCLR,all4one} to
train a neural network to learn the pairwise similarities. Once we train our
similarity network, we use the cosine of the representations of the samples as our
similarity function to either compute $\emW_{ij}=\exp\lrp{\frac{\cosine{z_i,
z_j}}{\tau}}$ for some temperature $\tau>0$ or to build the $k$-nearest neighbors
graph based on these representations.

\subsection{Computational and Memory Footprint}
If the similarity matrix $\mW$ is dense, the time complexity of computing the PRCut objective is $O(n^3)$, identical to a vanilla spectral clustering approach. In contrast, for a sparse graph where similarity is determined based on the $m$-nearest neighbors, the time complexity of conventional spectral clustering is $O(nmk)$, not accounting for the additional cost of $O(nk^2)$ incurred by running k-means on the computed eigenvectors.

In the case of batch PRCut, the expected computational cost for evaluating the batch loss with a batch size of $b$ is $O\left(\frac{m}{n}kb^2\right)$, while the memory requirement remains constant at $O(b^2)$. Consequently, when executed for $T$ steps, the overall time complexity of stochastic PRCut becomes $O\left(T\frac{m}{n}kb^2\right)$.
%
%
%

\section{EXPERIMENTS}
We train a neural network $ N_\theta: \sR^p \mapsto \Delta^{k-1} $ that maps the
vertex $v_i$ to its cluster assignment probabilities $ \mP^{\theta}_{i} $, where $
\theta \in \R^q $ is the parameter of the network. The last layer of the neural
network is a Softmax layer, ensuring that the constraint $ \sum_{\ell=1}^k
\emP_{i\ell} = 1 $ is always satisfied. We assume that the number of class labels
$k$ is provided. When computing the batch-based gradient
using~\Cref{eq:batchgradient}, we use the factor $\by{b}$ instead of $\by{n}$,
where $ b $ is the batch size\footnote{The code to reproduce PRCut is available at \href{https://github.com/ayghri/prcut}{\texttt{https://github.com/ayghri/prcut}}}.

Since $ \gL_{rc}(\mW, \mP^\theta) $ is linear in
$\mW$, we scale it by $\frac{1}{\normtwo{\mW}_1} $ to ensure consistency in the
gradient descent method across different datasets and similarity measures.

\paragraph{Metrics} To benchmark the performance of our algorithm, we assume that we have access to the
true labeling $\vy=(y_i)_i$ and the algorithm's clustering $\vc$. We use three
different metrics defined as follows:

\begin{itemize}
	\item \textbf{Unsupervised Accuracy (ACC)}: We use the \textit{Kuhn-Munkres} algorithm~\citep{munkres} to find the
	      optimal permutation$\sigma_k$ of $\{1,\ldots,k\}$ such that $\by{n}\max_{\sigma\in
			      \sigma_k} \sum_{i=1}^{n} 1_{y_i = \sigma(c_i)}$ is maximized between the
	      true labeling $(y_i)_i$ and the clustering $(c_i)_i$.

	\item \textbf{Normalized Mutual Information (NMI)} is defined as:
	      $\text{NMI}(\vy, \vc) = \frac{\mathcal{I}(\vy, \vc)}{\max\{\mathcal{H}(\vl),\mathcal{H}(\vc)\}}$
	      Where $\mathcal{I}(\vy, \vc)$ is the mutual information between $\vy$ and $\vc$ and
	      $\gH$ is the entropy measure.
	\item \textbf{Adjusted Rand Index (ARI)} to evaluates the agreement between
	      the true class labels and the learned clustering.
	\item \textbf{Ratio Cut (RC)} is computed using the formula in~\Cref{eq:rawlap}
	      using the raw Laplacian matrix.

\end{itemize}

Before we compare our method to spectral clustering and other clustering
approaches, we first assess the quality of the clustering when using the perfect
similarity measure, where $\gK(v_i,v_j)=1$ if $v_i$ and $v_j$ share the same label,
and $\gK(v_i,v_j)=0$ otherwise. For this experiment, we employ a simple Multi-Layer
Perceptron (MLP) consisting of 3 layers with 512 hidden units, followed by Gaussian
Error Linear Units (GeLU) activations. In the PRCut network, the last layer
utilizes the Softmax activation. We compare the two methods across three datasets:
MNIST~\citep{mnist}, Fashion-MNIST (F-MNIST)~\citep{fmnist}, and
CIFAR10~\citep{cifar}. The classifier is trained by optimizing the cross-entropy
loss (CE).

\begin{table}[ht]
	\caption{Benchmarking PRCut when using label-based similarity}
	\centering
	\begin{tabular}{l@{\hspace{5mm}}l@{\hspace{5mm}}ll}
		\toprule
		{Dataset}                & {Method} & ACC            & NMI            \\
		\hline
		\multirow{2}{*}{MNIST}   & CE       & 0.980          & \textbf{0.943} \\
		                         & PRCut    & \textbf{0.987} & 0.938          \\
		\hline
		\multirow{2}{*}{F-MNIST} & CE       & 0.885          & \textbf{0.803} \\
		                         & PRCut    & \textbf{0.887} & 0.789          \\

		\hline
		\multirow{2}{*}{CIFAR10} & CE       & \textbf{0.582} & \textbf{0.369} \\
		                         & PRCut    & 0.571          & 0.359          \\
		\hline
	\end{tabular}
\end{table}

The MLP classifier trained directly on the labeled data and the PRCut clustering
using the labeled similarity demonstrate similar performance. This shows that our
approach is more versatile and can be used in various learning methods while
perfectly reflecting the quality of the similarity in the resulting clustering.

In \Cref{table:exp2} we compare our approach to vanilla Spectral Clustering (SC)
using k-neighbor graph adjacency as a similarity measure for $k=150$ using the
entire training set. Since the final performance depends on the k-means
initialization, we report the best run for SC.

\begin{table}[ht]
	\caption{Comparison between PRCut and Spectral Clustering (SC)}
	\centering
	\begin{tabular}{l@{\hspace{5mm}}l@{\hspace{5mm}}lll}
		\toprule
		{Dataset}                & {Method}  & ACC            & NMI            & RC             \\
		\hline
		\multirow{3}{*}{MNIST}   & SC (Best) & 0.70           & 0.744          & 170.1          \\
		                         & PRCut     & \textbf{0.821} & \textbf{0.778} & \textbf{150.2} \\
		\hline
		\multirow{2}{*}{F-MNIST} & SC (Best) & 0.596          & 0.593          & 110.2          \\
		                         & PRCut     & \textbf{0.658} & \textbf{0.620} & \textbf{101.5} \\
		\hline

		\multirow{2}{*}{CIFAR10} & SC (Best) & 0.217          & 0.086          & 479.3          \\
		                         & PRCut     & \textbf{0.243} & \textbf{0.121} & \textbf{440.3} \\
	\end{tabular}
	\label{table:exp2}
\end{table}

Since the final objective is to minimize the ratio cut, our approach achieves a
better ratio cut value compared to vanilla spectral clustering. Whether such
improvement translates to an enhancement in the clustering metrics depends on the
similarity function. In all the 3 datasets, our approach outperforms the spectral
relaxation of the ratio-cut.

For the third set of experiments reported in \Cref{table:exp3}, we compare our
method to VaDE~\citep{vade}, VMM~\citep{VMM}, and Turtle~\citep{turtle}. We have
observed that VaDE does not generalize well across datasets, as its performance
drastically degrades when applied to the Fashion-MNIST dataset. The VMM approach is
a more consistent variational autoencoder (VAE) with a mixture model prior,
achieving the best reported performance on the Fashion-MNIST dataset in the
literature. We also report the results for methods with the suffix "-D," which use
the pre-trained representation model DINOv2~\citep{dinov2}, while the suffix "-V"
denotes methods that utilize the representation from the CLIP~\citep{clip} vision
transformer. In both of these approaches, the k-neighbors graph and the trained
neural networks take the samples in the representation spaces as input.

The neural network consists of a single linear layer followed by the softmax operator. The learning rate was set to a constant value of $10^{-4}$ and using a $10^{-7}$ weight decay, with a bach size of $b=2048$. We set the entropy regularization to $\gamma=100.0$ and the moving average parameter to $\beta=0.8$.

\begin{table}[ht]
	\caption{Comparison of PRCut to the best-performing clustering methods.}
	\centering
	\begin{tabular}{l@{\hspace{5mm}}l@{\hspace{5mm}}ll}
		\toprule
		{Dataset}                 & {Method} & ACC            & NMI          \\
		\hline
		\multirow{3}{*}{MNIST}    & SC       & 0.701          & 0.744        \\
		                          & VaDE     & 0.857          & 0.838          \\
		                          & VMM      & \textbf{0.960} & \textbf{0.907} \\
		                          & Turtle-D & 0.573          & 0.544          \\
		                          & PRCut-V  & 0.771           & 0.734         \\
		\hline
		\multirow{3}{*}{F-MNIST}
		                          & SC       & 0.596          & 0.593          \\
		                          & VaDE     & 0.352          & 0.496          \\
		                          & VMM      & 0.712          & 0.688          \\
		                          & Turtle-D & 0.764          & 0.723          \\
		                          & PRCut-D  & \textbf{0.791} & \textbf{0.758} \\
		\hline

		\hline
		\multirow{3}{*}{CIFAR10}  & SC       & 0.217          & 0.121          \\
		                          & Turtle-V & 0.972          & 0.929          \\
		                          & PRCut-V  & \textbf{0.975} & \textbf{0.934} \\
		\hline
		\multirow{3}{*}{CIFAR100} & Turtle-D & \textbf{0.806} & \textbf{0.870} \\
		                          & PRCut-D  & 0.789          & 0.856          \\
		\hline
	\end{tabular}
	\label{table:exp3}
\end{table}

Our method remains competitive compared to Turtle and achieves the best reported
performance on Fashion-MNIST, which was designed to be a more challenging dataset
than MNIST. Furthermore, our approach does not rely on any assumptions about the
structure of the representation spaces, in contrast to Turtle, which is based on
the premise that the best clusters are linearly separable—a property that is
inherently valid for DINOv2 and CLIP. This may explain why it does not perform as
well on grayscale images such as MNIST and Fashion-MNIST, where linear separability
is not as prominent. We note that the Turtle results in our benchmark are based on
a single representation space for a fair comparison, as the original paper performs
best when it combines two representation spaces.

\begin{table}[ht]
	\caption{Comparison representations using PRCut}
	\centering
	\begin{tabular}{l@{\hspace{5mm}}l@{\hspace{5mm}}lll}
		\toprule
		{Dataset}               & {Rep}   & ACC            & NMI            \\
		\hline
		\multirow{3}{*}{CIFAR10} & Raw    & 0.243          & 0.121          \\
		                          & SimCLR  & 0.721          & 0.652          \\
		                          & All4One & 0.710          & 0.635          \\
		                          & VitL-14 & \textbf{0.975} & \textbf{0.934} \\
		                          & DinoV2  & 0.774          & 0.797          \\
		\hline
		\multirow{3}{*}{CIFAR100} & Raw   & 0.054          & 0.022          \\
		                          & SimCLR  & 0.362          & 0.483          \\
		                          & All4One & 0.382          & 0.511          \\
		                          & VitL-14 & 0.720          & 0.755          \\
		                          & DinoV2  & \textbf{0.789} & \textbf{0.856} \\
		\hline
	\end{tabular}
	\label{table:exp4}
\end{table}

In \Cref{table:exp4}, we compare the performance of our method using various
pre-trained self-supervised representation learning models. In particular, we rely
on the \textit{solo-learn} library~\citep{sololearn} to retrieve or fine-tune the
SimCLR and All4One models. While these two approaches perform well when evaluated
using linear probes or a k-nearest neighbor classifier, the similarities in the
embedding space for datasets with a higher number of class labels (CIFAR100) are
too sparse to capture useful clusters. In that regard, DinoV2 embedding performs
better than CLIP vision transformer.

\section{Conclusion and future work}

We have introduced a novel method that approaches the graph ratio-cut optimization
from a probabilistic perspective. Compared to the classical Spectral Clustering
based on the raw Laplacian, PRCut achieves better clustering performance and more
optimal ratio-cut values. However, the performance strongly depends on the sparsity
of the global similarity. With the recently developed self-supervised
representation models that have proven to be powerful, we have demonstrated that
PRCut translates the similarities between samples in the embedding space into high
quality clustering and achieve new best clusterings for Fashion-MNIST while
remaining competitive with state-of-the-art clustering methods.

While the current work serves as an introduction to this novel approach, the
potential extensions of our method appear boundless. For example, PRCut could be
leveraged as a dimensionality reduction technique by incorporating a bottleneck
layer within the clustering neural network. Furthermore, the methodology could be
adapted under the premise of linear separability of the true class labels, similar
to the approach adopted by~\citet{turtle}. The issue of dynamically determining the
optimal number of clusters remains unresolved, given our current assumption that
this information is supplied as an input to our algorithm.

The algorithm can also be extended to offline learning via~\Cref{eq:quad}. Notably,
when considering equally likely clusters, the problem simplifies to a quadratic
form, with the Hessian being proportional to $-\mW$. By leveraging the structure of
the similarity matrix, we can derive additional guarantees about the optimal
solution. This enables the definition of an iterative approach similar to the
Sinkhorn-Knopp algorithm~\citep{sinkhorn}, offering promising avenues for further
exploration.

\bibliography{aistats2025_paper}
\bibliographystyle{iclr2025}
\clearpage
\onecolumn
\appendix
\section{Proofs for the upper bound}
\subsection{Ratio-cut expectation proof}
\label{appendix:diffexpect}
\diffexpect*
\begin{proof}
	Using the probabilistic formulations. We have:
	\[
		\widehat{|C|} = \sum_{i=1}^{n} \rva^{}_i = \rva_1  + \rva_2
		+ \widehat{|C^{\perp(1,2)}|}
	\]
	We can then expand the expectation:
	\begin{align*}
		\Ea{(\widehat{\rvf}_1-\widehat{\rvf}_2)^2}
		 & = \Ea{\by{\rva_1+\rva_2 + \widehat{|C^{\perp(i,j)}|}}\lrp{\rva_1-\rva_2}^2} \\
		 & = \Ea{\by{1+\widehat{|C^{\perp(1,2)}|}}} \Ppar{\rva_1,\rva_2 = 1,0}
		+  \Ea{\by{1+\widehat{|C^{\perp(1,2)}|}}} \Ppar{\rva_1,\rva_2 = 0,1}           \\
		 & = \Ea{\by{1+\widehat{|C^{\perp(1,2)}|}}}\lrp{p_1 (1-p_2) + (1-p_1)p_2}      \\
		 & = \Ea{\by{1+\widehat{|C^{\perp(1,2)}|}}}\lrp{p_1+p_2-2p_1p_2}
	\end{align*}
\end{proof}
\subsection{Binomial Poisson distribution properties} \label{appendix:poissonprop}
Let $\rZ$ be a Poisson binomial random variable of parameter
$\vp=(\alpha_1,\ldots,\alpha_n)$. We can write $\rZ$ as:
\[
	\rZ=\sum_{i=1}^m \rvr_i,
\]
where $(\rvr_i)_i$ are independent random variable such that:
\[
	\rvr_i \sim \text{Bernoulli}(\alpha_i).
\]
\begin{restatable}{lemma}{poissonintro}
	\label{lemma:poissonintro}
	We denote by $G_Z$ the Probability-Generating Function (PGF) for the Poisson binomial
	$\rZ$ of parameter $\bm\alpha\in[0,1]^m$. We have the following:
	{\footnotesize
	\begin{align*}
		\mu      & \defeq \Ea{\rZ} = \sum_{i=1}^n \alpha_i                     \\
		\sigma^2 & \defeq \Ea{\lrp{\rZ-\mu}^2} = \mu - \sum_{i=1}^m \alpha^2_i \\
		G_Z(t)   & \defeq \Ea{t^\rZ} = \prod_{i=1}^m (1-\alpha_i+\alpha_i t)
	\end{align*}
	}
\end{restatable}

\begin{proof}
	\begin{align*}
		\mu & \defeq \Ea{\rZ} = \sum_{i=1}^m\Ea{ \rvr_i} = \sum_{i=1}^n \alpha_i \\
	\end{align*}
	Since $(\rvr_i)$ are independent random variable, the variance of their sum is the
	sum of the variance:
	\begin{align*}
		\sigma^2 & \defeq \text{Var}[Z] = \sum_{i=1}^n\text{Var}[\rvr_i] = \sum_{i=1}^n
		\alpha_i(1-\alpha_i)=
		\mu - \sum_{i=1}^n \alpha^2_i                                                   \\
	\end{align*}
	For the PGF:
	\begin{align*}
		G_Z(t)  \defeq \Ea{t^Z} & = \Ea{t^{\sum_{i=1}^n \rvr_i}}                                                 \\
		                        & = \prod_{i=1}^n \Ea{t^{\rvr_i}} \prod_{i=1}^n (1-\alpha_i) * t^0 +\alpha_i t^1 \\
		                        & = \prod_{i=1}^n (1-\alpha_i +\alpha_i t)
	\end{align*}
\end{proof}
The first approach for computing $\Ea{\by{1 + Z}}$ is to use the definition of the
expectation, which requires the knowledge of the probabilities $\lrp{\Ppar{Z=i}}_i$:
\[
	\Ea{\by{1 + Z}} = \sum_{i=0}^n \by{1+i} \Ppar{Z=i}.
\]
\begin{lemma}
	\label{lemma:probas}
	Let $\sP\lrb{n}$ denote the power set of $\lrcb{1,\ldots,n}$, and
	$\gI_i$ be the set of elements of $\sP\lrb{n}$ that contain exactly $i$
	unique integers. $\Ppar{Z=i}$ can be computed as:
	\begin{align}
		\Ppar{Z=i} = \sum_{I \in \gI_i} \prod_{j\in I}p_j\prod_{m\in \overline{I}}(1-p_m)
	\end{align}
\end{lemma}
\begin{proof}
	Refer to \cite{poissonbinomial}
\end{proof}

\subsection{Proof of the integral formula}
\label{appendix:poissonexpect}

We will use the properties introduced in \Cref{appendix:poissonprop} to prove the
following lemma: \poissonexpect*
\begin{proof}
	The PGF for $\rZ$ is
	\[
		G_\rZ(t) = \Ea{t^\rZ} = \prod_{i=1}^m (1-\alpha_i+\alpha_i t)
	\]
	Using the Probability-Generating Function (PGF), we can write:
	\begin{align*}
		\Ea{\by{1 + Z}}
		 & =\Ea{\int_{0}^{1} t^Z dt}
		= \int_{0}^{1} \Ea{t^Z}dt
		 &                                                                                                 \\
		 & = \int_{0}^{1} \prod_i (1-\alpha_i+\alpha_i t)dt &                                              \\
		 & = \int_{0}^{1} \prod_i (1-\alpha_i t)dt          & \text{change of variable: }  t\leftarrow 1-t \\
	\end{align*}

	Exchanging the expectation and the integral is based on the fact that the
	function we are integrating is measurable, non-negative, and bounded by the
	constant \(1\). We can then use Fubini's theorem to change the order of
	integration.
\end{proof}

\begin{corollary}
	Using the notations from \cref{lemma:probas} we can show that:
	\begin{align*}
		\Ea{\by{1 + Z}}  = \sum_{i=1}^{n} \frac{(-1)^i}{1+i} \sum_{I\in\gI_i} \prod_{j\in I} p_j,
	\end{align*}
	where $\gI_i$ is the set of subsets of $\lrcb{1,\ldots,n}$ that contain exactly
	$i$ elements.
\end{corollary}

This offers a straightforward approach for approximating the expectation, which is
notably simpler to compute iteratively. However, it's not practical as the number
of terms to sum is $n!$ and can only be implemented for small graphs.

\subsection{Integral computation}
\label{appendix:integralcompute}
\integralcompute*
\begin{proof}
	This theorem is a direct application of Gauss-Legendre quadrature (See~\citet{numericalanalysis}). It
	states that any integral of the form:
	\[
		\int_{-1}^{1} f_m(t)dt
	\]
	where $f_m$ is polynomial of degree at most $2m-1$ can be exactly computed
	using the sum:
	\[
		\sum_{i=1}^{m} w_i f(r_i)
	\]
	where $(r_i)$ are the roots of $P_m$, Legendre polynomial of degree $m$, and
	the weights are computed as:
	\[
		w_i = \frac{2}{(1-r_i)^2 [P_n'(r_i)]^2}\geq 0
	\]
	For an integral on the $[0,1]$ interval, we use change of variable to find the
	corresponding weights and nodes are:
	\begin{align*}
		s_j & = \frac{w_j}{2}          \\
		t_j & = \frac{r_j}{2} + \by{2}
	\end{align*}
\end{proof}

\subsection{Batch estimation of the expected ratio-cut}\label{appendix:batchestimate}

Returning to the probabilistic ratio-cut, we can combine
corollary~\cref{thm:ratiocutexp} and \cref{lemma:poissonexpect} to express the
expected ratio-cut as:
\begin{align*}
	\sum_{i=1}^{n}\sum_{j>i}^{n} w_{i,j}(p_i + p_j  -2 p_i p_j)
	\int_{0}^{1} \prod_{k\neq i,j} (1-p_k t)dt
\end{align*}

Practically, with large datasets, it is expensive to compute the quantity
$\int_{0}^{1} \prod_{k\neq i,j} (1-p_k t)dt$ on the entire dataset: the quadrature
calculation costs $\mathcal{O}(\frac{n^2}{2})$ and the total cost of computing the
objective is $\mathcal{O}(k n^3)$ for all the clusters.

We can also show that we can get a tighter bound if $S$ is a random subset such
that
\[
	\forall i\in[1,n] P(i\in S) = \gamma
\]

\begin{restatable}[Batch estimation]{thm}{batchbound}
	\label{thm:batchbound}
	Let $S$ is a random subset such
	that
	\[
		\forall i\in[1,n], P(i\in S) = \gamma < 1.
	\]
	Then we can show that:
	\begin{align}
		\int_{0}^{1} \prod_{i=1}^{n} (1-p_k t)dt \leq
		\Eb{S}{\int_{0}^{1}\prod_{s\in S} \lrp{1- p{s}t}^{\by{\gamma}}dt}
	\end{align}
\end{restatable}
\begin{proof}

	\begin{align*}
		\prod_{i=1}^n (1- p_{i}t)
		 & = \exp\lrb{\sum_{i=1}^n \log(1- p_i t)}                       \\
		 & = \exp\lrb{\by{\gamma} \Eb{S}{\sum_{i\in S}\log(1- p_i t)}}   \\
		 & = \exp\lrb{\Eb{S}{\sum_{i\in S}\by{\gamma} \log(1- p_i t)}}   \\
		 & \leq \Eb{S}{\exp\lrb{\sum_{i\in S}\by{\gamma}\log(1- p_i t)}} \\
		 & = \Eb{S}{\prod_{s\in S} (1- p_st)^\by{\gamma}}.
	\end{align*}

	\Cref{thm:batchbound} provides a bound on the integral based on the
	expectation of a random batch of samples. We can show that if $p_i\in\lrcb{0,1}$ that
	the expectation over the batch equal to the estimate over the whole dataset. If we sample
	a random batch of size $b$ then $\gamma=\frac{b}{n}$, the fraction of the entire data
	that we're using.
\end{proof}
\subsection{Objective upper bound}\label{appendix:integralupperbound}
\integralupperbound*
\begin{align*}
	\int_0^1 \prod_{i=1}^{m}(1-\alpha_i t)dt
	 & = \int_0^1 \exp \lrb{\sum_{i=1}^{m}\log(1-\alpha_i t)}dt                                              \\
	 & = \int_0^1 \exp \lrb{\by{m}\sum_{i=1}^{m}\log(1-\alpha_i t)}^m dt                                     \\
	 & \leq \int_0^1 \lrb{ \by{m} \sum_{i=1}^{m} \exp\log(1-\alpha_i t)}^m dt & \text{exponential is convex} \\
	 & =  \int_0^1 \lrb{\by{m}  \sum_{i=1}^{m} (1-\alpha_i t)}^m dt                                          \\
	 & =  \int_0^1 \lrp{1-\ov{p} t}^m dt                                                                     \\
	 & =  -\by{\ov{\bm\alpha}(m+1)} \lrb{\lrp{1-\ov{\alpha}t}^{m+1}}_0^1                                     \\
	 & =  \by{\ov{\bm\alpha}(m+1)} \lrb{1-\lrp{1-\ov{\alpha}}^{m+1}}                                         \\
	 & \leq \by{\ov{\bm\alpha}(m+1)}.
\end{align*}

Moving from the proof of \Cref{lemma:integralupperbound} to \Cref{lemma:rcupper}
requires a modification in the previous proof. To get an upper bound without making
any assumptions about $\ov{\vp}$, we can prove the following:
\begin{align*}
	\sI(\vp, 1, 2) & = \int_{0}^{1} \prod_{i\geq 3}^{n} (1-\evp_i t)dt                    \\
	               & \leq \int_0^1 \lrb{\by{n-2}  \sum_{i=3}^{n} (1-\evp_i t)}^{(n-2)} dt \\
	               & \leq \int_0^1 \lrb{\by{n-2}  \sum_{i=1}^{n} (1-\evp_i t)}^{(n-2)} dt \\
	               & = \int_0^1 \lrb{\frac{n}{n-2} (1-\ov{\vp} t)}^{(n-2)} dt             \\
	               & = \lrp{\frac{n}{n-2}}^{n-2}\by{(n-1)\ov{\vp}}                        \\
	               & = \lrp{\frac{n}{n-2}}^{n-2}\frac{n}{n-1}\by{n\ov{\vp}}               \\
	               & = (e^2-o(n))\by{n\ov{\vp}}                                           \\
                   & < \frac{e^2}{n\ov{\vp}}                                           \\
\end{align*}

But by assuming that $\by{n-2}\sum_{i=3}^{n} (1-\evp_i t) \approx
\by{n}\sum_{i=1}^{n} (1-\evp_i t)=\ov{\vp}$, we get the bound in
\Cref{lemma:rcupper}:
\begin{align*}
	\by{n-1}\by{\by{n-2}\sum_{i=3}^{n} \evp_i}
	 & = \frac{n-2}{n-1} \by{\sum_{i=3}^{n} \evp_i}       \\
	 & \approx \frac{n-2}{n-1} \by{\sum_{i=1}^{n} \evp_i} \\
	 & \leq \by{n}\by{\ov{\vp}}
\end{align*}

\section{Proofs for the stochastic gradient}
\subsection{Offline Gradient} \label{appendix:offlinegrad}
The offline gradient of the objective can be obtained directly by differentiating
the matrix version of the expression
\begin{align*}
	\frac{d \gL_{rc}}{d\mP}
	= \ov{\mP}^{-1}\lrb{\lrp{\vone_{n,k}-\mP}^\top\mW\vone_{n,k}-\vone_{n,k}^\top\mW\mP}
	-\by{n}\vone_{n,k}\lrb{\ov{\mP}^{-1}(\vone_{n,k}-\mP)^\top\mW \mP\ov{\mP}^{-\top}}
\end{align*}

This expression is particularly useful in the offline learning setting, especially
when $\mW$ is sparse, as it makes the computation graph more efficient due to the
limited support of sparse gradients in major deep learning libraries.
\subsection{Online Gradient}\label{appendix:onlinegrad}
To compute the derivative of with respect to $p_i$, we simply calculate the
derivatives of $\by{\ov{\vp}}$ and $(p_i + p_j -2 p_i p_j)$ w.r.t $p_i$ which are:
\begin{align*}
	\frac{d}{d p_i}\by{\ov{\vp}}           & = -\by{n\ov{\vp}^2} \\
	\frac{d}{d p_i}(p_i + p_j  -2 p_i p_j) & = (1-2p_j),         \\
\end{align*}
respectively.

By plugging these derivatives and considering the rule for the derivative of the
product $(fg)' = f'g + fg'$ of two functions $f$ and $g$, the result in
\Cref{eq:batchgradient} follows.

\section{Reproducibility method}\label{appendix:experiments}

\subsection{ Experiments setup}
In our sets of experiments, there are 2 variants:
\begin{itemize}
	\item \textbf{Original representation space}, where the PRCut algorithm is used to train a neural
	      network that takes the original as input
	\item \textbf{SSL representation space}, where the neural network takes as input
	      the transformed dataset via a pretrained self-supervised model
\end{itemize}

\subsubsection{Datasets}
We rely on the following datasets to benchmark our method:
\paragraph{MNIST}
MNIST~\citep{mnist} is a dataset consisting of 70,000 28x28 grayscale images of
handwritten digits, which are split into training (60,000) and test (10,000)
samples.
\paragraph{Fashion MNIST}
Fashion-MNIST~\citep{fmnist} is a dataset of Zalando's article images, comprising a
training set of 60,000 examples and a test set of 10,000 examples. Each example is
a 28x28 grayscale image linked to a label from 10 classes. Zalando designed
Fashion-MNIST to be a direct substitute for the original MNIST dataset in
benchmarking machine learning algorithms. However, it introduces more complexity to
machine learning tasks due to the increased difficulty in classifying the samples.

\paragraph{CIFAR}
The CIFAR-10 and CIFAR-100 datasets\citep{cifar} is a widely used benchmark for
image classification tasks in computer vision. It consists of 60,000 32x32 color
images, divided into 100 classes, with 600 images per class. The images are grouped
into 20 superclasses.

Key Characteristics: 60,000 images (50,000 for training and 10,000 for testing)
32x32 pixel resolution 100 classes, with 600 images per class 20 superclasses,
grouping the 100 classes

\subsubsection{Encoder Architecture}
The encoder is a simple stack of linear layers followed each by a GeLU activation.
The first layer and the last layer are always weight-normalized. The last
activation is a softmax layer. The hidden unit used is constant across layers.

\subsubsection{Hyperparameters}

\begin{table}[ht!]
	\caption{Hyperparameters}\label{tab:}
	\begin{center}
		\begin{tabular}[c]{|l|l|}
			\hline
			{Hyperparameter}                & Value     \\
			\hline
			$\beta$ (moving average)       & 0.8       \\
			\hline
			$\gamma$ (regularization weight) & 100.0     \\
			\hline
			Optimizer original space        & RMSProp   \\
			\hline
			Optimizer representation space  & Adam      \\
			\hline
			Learning rate                   & $10^{-4}$ \\
			\hline
			Weight decay                    & $10^{-7}$ \\
			\hline
			k (k neighbors graph)           & $100$     \\
			\hline
		\end{tabular}
	\end{center}
\end{table}

\end{document}